\documentclass[letterpaper, 10 pt, conference]{ieeeconf} 

\IEEEoverridecommandlockouts                            
                                                          
\usepackage[utf8]{inputenc}
\usepackage[T1]{fontenc}
\usepackage{amsmath,amssymb,amsfonts}
\usepackage{times}
\usepackage{graphicx}
\usepackage{subcaption}
\usepackage{epsfig}
\usepackage{cite}
\usepackage{hyperref}
\usepackage{color}
\usepackage{mathtools}
\usepackage{bm}
\usepackage{algorithm}
\usepackage{algpseudocode}
\usepackage{url}
\usepackage{tikz}
\usepackage[percent]{overpic}
\usetikzlibrary{positioning, arrows.meta}
\usepackage{xcolor}

\newtheorem{lemma}{\textbf{Lemma}}
\newtheorem{theorem}{\textbf{Theorem}}

\title{\LARGE \bf
Barometer-Aided Attitude Estimation
}

\author{Méloné Nyoba Tchonkeu$^{1}$, Soulaimane Berkane $^{2}$, \textit{Senior Member, IEEE},\\ and Tarek Hamel$^{3}$, \textit{Fellow Member, IEEE}%
\thanks{*This work was supported by the"Grands Fonds Marins" Project Deep-C, and the ASTRID ANR project ASCAR. This research work is also supported in part by NSERC-DG RGPIN-2020-04759 and Fonds de recherche du Québec (FRQ).}%
\thanks{$^{1}$M. Nyoba Tchonkeu is with the Department of Computer Science and Engineering, University of Quebec in Outaouais, 
        Gatineau, QC J8X3X7, Canada
        {\tt\small (nyom01@uqo.ca)}}%
\thanks{$^{2}$S. Berkane is with the Department of Computer Science and Engineering, University of Quebec in Outaouais, Gatineau, QC J8X3X7, and also with the Department of Electrical Engineering, Lakehead University,
        Thunder Bay, ON P7B 5E1, Canada
        {\tt\small (soulaimane.berkane@uqo.ca)}}%
\thanks{$^{3}$T. Hamel is with I3S-UniCA-CNRS, University Cote d’Azur and the Insitut Universitaire de France,
        06903 Sophia Antipolis, France
        {\tt\small (thamel@i3s.unice.fr)}}%
}

\begin{document}

\maketitle
\thispagestyle{empty}
\pagestyle{empty}

%%%%%%%%%%%%%%%%%%%%%%%%%%%%%%%%%%%%%%%%%%%%%%%%%%%%%%%%%%%%%%%%%%%%%%%%%%%%%%%%
\begin{abstract}
Accurate and robust attitude estimation is a central challenge for autonomous vehicles operating in GNSS-denied or highly dynamic environments. In such cases, Inertial Measurement Units (IMUs) alone are insufficient for reliable tilt estimation due to the ambiguity between gravitational and inertial accelerations. While auxiliary velocity sensors, such as GNSS, Pitot tubes, Doppler radar, or visual odometry, are often used, they can be unavailable, intermittent, or costly.
This work introduces a barometer-aided attitude estimation architecture that leverages barometric altitude measurements to infer vertical velocity and attitude within a nonlinear observer on $\mathrm{SO}(3)$. The design cascades a deterministic Riccati observer with a complementary filter, ensuring Almost Global Asymptotic Stability (AGAS) under a uniform observability condition while maintaining geometric consistency. The analysis highlights barometer-aided estimation as a lightweight and effective complementary modality.

\end{abstract}

%%%%%%%%%%%%%%%%%%%%%%%%%%%%%%%%%%%%%%%%%%%%%%%%%%%%%%%%%%%%%%%%%%%%%%%%%%%%%%%%
\section{INTRODUCTION}
Accurate attitude estimation for autonomous vehicles in GNSS-denied or highly dynamic environments remains a core challenge in inertial navigation. The IMU, which provides angular velocity and specific acceleration, serves as the primary sensing backbone. In many estimation frameworks (\textit{e.g.,}\cite{mahony2008nonlinear,hua2013implementation}), especially for low-cost UAVs and robotic systems, gravity is approximated using accelerometer measurements under the assumption of negligible linear accelerations. Yet, accelerometers do not measure gravity directly but the specific acceleration, which includes all non-gravitational accelerations. This approximation holds only under quasi-static conditions, where gravity is the dominant acceleration. In dynamic scenarios, such as legged locomotion, aggressive UAV maneuvers, or strong wind gusts, the signal deviates from the gravity direction, leading to ambiguity in tilt estimation. To address this, modern estimation architectures increasingly incorporate additional measurements to enhance observability and robustness.

Several authors have incorporated either body-frame magnetometers or inertial-frame velocity measurements to address this issue. This so-called velocity-aided attitude (VAA) problem has recently attracted increasing attention \cite{hua2010attitude,roberts2011attitude,berkane2017attitude}. Early solutions to the body-frame VAA problem relied on linearisation, e.g., \cite{2008_bonnabel_SymmetryPreservingObservers,2008_martin_InvariantObserverEarthVelocityAided}, whereas later approaches adopted more constructive designs and, in some cases, provided guarantees of almost-global asymptotic stability \cite{2013_troni_PreliminaryExperimentalEvaluation,2016_hua_StabilityAnalysisVelocityaided,wang2021nonlinear,benallegue2023velocity,Pieter2023}.
Although robust and theoretically grounded, these architectures generally assume full measurement of either body or inertial velocity, which limits their practical deployment when only partial velocity information is available or when measurements are unreliable. This limitation has received little theoretical attention. A notable exception is the work of Oliveira \textit{et al.} \cite{oliveira2024pitot} that addresses tilt and air velocity estimation for fixed-wing UAVs in GNSS-denied conditions by exploiting only a component of the air velocity vector in the body fixed frame provided by a \textit{single-axis Pitot tube}  and IMU data within a Riccati observer framework on $\mathrm{SO}(3) \times \mathbb{R}^3$ and guarantees local asymptotic convergence. While relevant, the approach is limited to fixed-wing UAVs and prone to uncertainties in Pitot tube measurements. 

In this paper, we focus on the contribution of a more versatile barometer sensor in estimating attitude. Barometers, widely available in modern UAVs and mobile robots, provide continuous altitude readings from which vertical velocity can, in principle, be obtained by temporal differentiation. This vertical velocity constrains motion along the gravity axis, offering information analogous to that of airspeed sensors for aiding attitude estimation. To address the noise inherent in barometric measurements, we first design a Riccati observer that jointly estimates vertical velocity and the specific acceleration in the tilt vector, ensuring global uniform exponential stability of the estimation error. Building on this, we then derive a second, cascaded filter to recover the full orientation in $\mathrm{SO}(3)$ with almost global stability guarantee. The method extends the state of the art and introduces an underexplored sensor modality to the field of attitude estimation. More importantly, it improves robustness in \textit{GNSS-degraded} scenarios and supports lightweight, redundant sensing architectures for embedded platforms. 

\section{PRELIMINARY MATERIAL}
\label{sec:prelims}

We denote by \( \mathbb{R} \) and \( \mathbb{R}_+ \) the sets of real and nonnegative real numbers, respectively. The \( n \)-dimensional Euclidean space is denoted by \( \mathbb{R}^n \). The Euclidean inner product of two vectors \( x, y \in \mathbb{R}^n \) is defined as \( \langle x, y \rangle = x^\top y \). The associated Euclidean norm of a vector \( x \in \mathbb{R}^n \) is \( |x| = \sqrt{x^\top x} \). Furthermore, we denote by \( \mathbb{R}^{m \times n} \) the set of real \( m \times n \) matrices. The set of \( n \times n \) positive definite matrices is denoted by \( \mathcal{S}^+(n) \), and the identity matrix is denoted by \( I_n \in \mathbb{R}^{n \times n} \). Given two matrices \( A, B \in \mathbb{R}^{m \times n} \), the Euclidean matrix inner product is defined as \( \langle A, B \rangle = \mathrm{tr}(A^\top B) \), and the Frobenius norm of \( A \in \mathbb{R}^{n \times n} \) is given by \( \|A\| = \sqrt{\langle A, A \rangle} \). We denote by \( A \otimes B \) the Kronecker product of matrices \( A \) and \( B \). The unit sphere \( \mathbb{S}^2 := \{ \eta \in \mathbb{R}^3 \mid |\eta| = 1 \} \subset \mathbb{R}^3 \) denotes the set of unit 3D vectors and forms a smooth submanifold of \( \mathbb{R}^3 \). A function \(g(x)\) is denoted \( \mathcal{O}(x) \) if it is bounded for any bounded \( x \in \mathbb{R}^n \), and \( g(x) \to 0 \) as \( x \to 0 \).

The special orthogonal group of 3D rotations is denoted by
$
\mathrm{SO}(3) := \{ R \in \mathbb{R}^{3 \times 3} \mid RR^\top = R^\top R = I_3,\ \det(R) = 1 \}.
$
The Lie algebra of $\mathrm{SO}(3)$ is 
$
\mathfrak{so}(3) := \{\, \Omega \in \mathbb{R}^{3\times 3} \mid \Omega^\top = -\Omega \,\},
$
isomorphic to $\mathbb{R}^3$ via the skew-symmetric operator 
$(\cdot)^\times : \mathbb{R}^3 \to \mathfrak{so}(3)$, defined such that 
$
u \times v = u^\times v$ for all $u,v \in \mathbb{R}^3.$
The exponential map \( \exp : \mathfrak{so}(3) \rightarrow \mathrm{SO}(3) \) defines a local diffeomorphism from a neighborhood of \( 0 \in \mathfrak{so}(3) \) to a neighborhood of \( I_3 \in \mathrm{SO}(3) \). This enables the composition map \( \exp \circ (\cdot)^\times : \mathbb{R}^3 \rightarrow \mathrm{SO}(3) \), which is given by the following Rodrigues' formula~\cite{Ma2004rodriguesformula} :

\begin{equation}
\exp([\theta]^{^\times}) 
= I_3 - \frac{\sin(\|\theta\|)}{\|\theta\|}[\theta]^{^\times}
+ \frac{1-\cos(\|\theta\|)}{\|\theta\|^2}([\theta]^{^\times})^2.
\label{eq:rodriguesformula}
\end{equation}

Consider the linear time-varying (LTV) system given by
\begin{equation}
\begin{cases}
\dot{x} = A(t)x + B(t)u, \\
y = C(t)x,
\end{cases}
\label{eq:LTV_system}
\end{equation}
with state \( x \in \mathbb{R}^n \), input \( u \in \mathbb{R}^\ell \), and output \( y \in \mathbb{R}^m \). The matrix-valued functions \( A(t) \), \( B(t) \), and \( C(t) \) are assumed to be continuous and bounded. By definition from~\cite{Besancon2007}, the system ~\eqref{eq:LTV_system} or pair \( (A(t), C(t)) \) is \emph{uniformly observable} if there exist constants 
\( \delta, \mu > 0 \) such that, for all \( t \geq 0 \),
\begin{equation}
W(t, t+\delta) := \frac{1}{\delta} \int_t^{t+\delta} 
\Phi^\top(s, t) \, C^\top(s) \, C(s) \, \Phi(s, t) \, ds 
\geq \mu I_n, \label{eq:observability_gramian}
\end{equation}
where \( \Phi(s, t) \) is the state transition matrix such that
\begin{equation}
\frac{d}{dt} \Phi(s, t) = A(t)\Phi(s, t), 
\quad \Phi(t, t) = I_n, \quad \forall s \geq t. \label{eq:transition_matrix} 
\end{equation}
%\]
\( W(t, t+\delta) \) is called the \emph{observability Gramian} of the system.

\section{PROBLEM DESCRIPTION}
\begin{figure}[ht]
  \centering
  \begin{overpic}[width=0.7\linewidth]{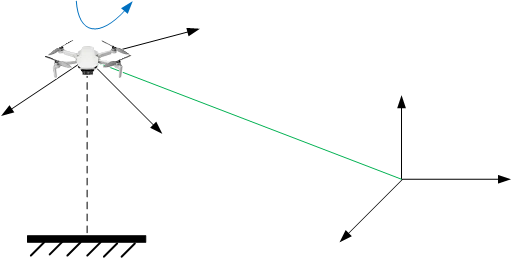}
  %\begin{overpic}[width=0.7\linewidth,grid, tics=10]{Vehicle_Equipped_IMU_Baro.png}
    \put(15,38){\small \(\left\{\mathcal{B}\right\}\)}
    \put(38,46){\small \(e_1^{\mathcal{B}}\)}
    \put(0,32){\small \(e_3^{\mathcal{B}}\)}
    \put(32,24){\small \(e_2^{\mathcal{B}}\)}
    \put(20,20){\small \(h\)}
    \put(80,17){\small \(\left\{\mathcal{I}\right\}\)}
    \put(80,30){\small \(e_3\)}
    \put(100,15){\small \(e_2\)}
    \put(66,8){\small \(e_1\)}
    \put(60,25){\small \(p\)}
    \put(26,48){\small \(R\)}
  \end{overpic}
  \caption{\small \textit{We assumed a vehicle equipped with an IMU and a barometer measuring the height $h$.
  }}
  \label{fig:veh_baro_imu}
\end{figure}
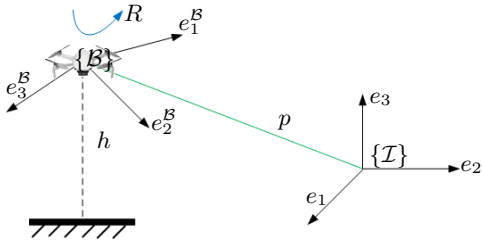
Let \( R \in \mathrm{SO}(3) \) denote the rotation matrix describing the orientation of the
body-attached frame \( \mathcal{B} \) with respect to the inertial frame \( \mathcal{I} \). The position and linear velocity of the rigid body, expressed in the inertial frame \( \mathcal{I} \), are denoted by \( p \in \mathbb{R}^3 \) and \( v \in \mathbb{R}^3 \), respectively.
The IMU components provide body-fixed measurements of the angular velocity \( \boldsymbol{\omega} \in \mathbb{R}^3 \) and the linear specific acceleration \( \mathbf{a} \in \mathbb{R}^3 \). The second-order kinematics of the vehicle are given by
\begin{align}
\dot{p} &= v, \label{eq:position_dyn} \\
\dot{v} &= R\mathbf{a} + \mathbf{g}, \label{eq:velocity_dyn} \\
\dot{R} &= R\boldsymbol{\omega}^\times,
\label{eq:attitude_dyn}
\end{align}
where \( \mathbf{g} = g \mathbf{e}_3 \in \mathbb{R}^3 \) is the gravitational acceleration expressed in the inertial frame, \( \mathbf{e}_3 = \begin{bmatrix}0  & 0 & 1 \end{bmatrix}^\top \in \mathbb{S}^2 \) denotes the standard gravity direction, and \( g \approx 9.81\, \text{m/s}^2 \) is the gravity constant. To enhance observability of the vertical motion, we define the altitude as 
$h := \mathbf{e}_3^\top p$. The barometric sensor output is modeled as
\begin{equation}
y_b = h + n_b,
\label{eq:baro_measurement}
\end{equation}
where $n_b \sim \mathcal{N}(0,\sigma_b^2)$ is zero-mean noise. 
This scalar measurement provides partial information on the inertial position and, through its time derivatives, can also contribute to constraining the gravity direction for tilt estimation, particularly when GNSS velocity data are unavailable. In addition, we assume available a magnetometer that provides measurements of the Earth's magnetic field,
\begin{equation}
\mathbf{m}_{\mathcal{B}} = R^\top \mathbf{m}_{\mathcal{I}},
\label{eq:mag_measurement}
\end{equation}
where $\mathbf{m}_{\mathcal{I}} \in \mathbb{S}^2$ is the known magnetic field vector expressed in the inertial frame. In summary, the vehicle is modeled by the dynamics 
\eqref{eq:position_dyn}--\eqref{eq:attitude_dyn} and equipped with an IMU and a barometric altimeter. 
The objective is to estimate the attitude $R \in SO(3)$ from these measurements, exploiting the altitude information to improve observability of the vertical direction.

\section{PROPOSED OBSERVER DESIGN}

This section presents a two-stage observer architecture for attitude estimation. The design exploits a reduced-order Riccati observer for vertical dynamics and tilt estimation, followed by a nonlinear observer on \( \mathrm{SO}(3) \) to estimate full orientation.

\subsection{Tilt and Barometric Altitude Estimation}

We define the tilt (reduced attitude) as the gravity direction expressed in the body frame~\cite{benallegue2023velocity}:
\begin{equation}
z := R^\top \mathbf{e}_3 \in \mathbb{S}^2,
\end{equation}
which evolves according to
\begin{equation}
\dot{z} = -\boldsymbol{\omega}^\times z.
\label{eq:z_dot}
\end{equation}
For the altitude $h := \mathbf{e}_3^\top p$, differentiation yields
\begin{equation}
\ddot{h} = g + \mathbf{a}^\top z,
\label{eq:h_ddot}
\end{equation}
where $\boldsymbol{\omega}$ and $\mathbf{a}$ are the IMU angular velocity and specific acceleration.  
Together, \eqref{eq:z_dot}--\eqref{eq:h_ddot} and the barometer measurement 
\eqref{eq:baro_measurement} can be cast into a linear time-varying (LTV) system. 
Letting $x := \begin{bmatrix} h & \dot{h} & z^\top \end{bmatrix}^\top \in \mathbb{R}^5$, we obtain
\begin{equation}
\begin{cases}
\dot{x} = A(t)x + Bg, \\
y_b = Cx + n_b,
\end{cases}
\label{eq:ltv}
\end{equation}
with
\begin{equation}
A(t) = 
\begin{bmatrix}
0 & 1 & \mathbf{0}_{1 \times 3} \\
0 & 0 & \mathbf{a}^\top \\
\mathbf{0}_{3 \times 1} & \mathbf{0}_{3 \times 1} & -\boldsymbol{\omega}^\times
\end{bmatrix}, \quad 
B = \begin{bmatrix} 0 \\ 1 \\ \mathbf{0}_{3 \times 1} \end{bmatrix}, \label{eq:state_command_matrix}
\end{equation}
and the output matrix :
\begin{equation}
C = \begin{bmatrix} 1 & 0 & \mathbf{0}_{1 \times 3} \end{bmatrix}. \label{eq:output_matrix}
\end{equation}

To estimate the state vector $x$ independently of the attitude $R$, we decouple its dynamics from the attitude estimation by designing a Riccati observer of the Kalman–Bucy type. The observer is given by
\begin{equation}
\dot{\hat{x}} = A(t) \hat{x} + B g + K(t)\left(y_p - C \hat{x} \right), \label{eq:riccati_observer}
\end{equation}
where $K(t) = P(t) C^\top M^{-1},$ with \(M= \sigma_b^2 \), and \(P(t)\) solution to the Continuous-time Riccati Equation (CRE): \(\dot{P} = A(t) P + P A^\top(t) - P C^\top M^{-1} C P + Q\), where \(P(0)\) is a positive definite matrix, and \( Q \succ 0 \) is a symmetric positive definite matrix modeling process uncertainty. The stability and convergence properties of this observer are fundamentally linked to the well-posedness of the CRE, and more precisely, to the UO of the system. The UO property guarantees the existence of a unique, bounded, and positive-definite solution \( P(t) \) to the CRE for all \( t \ge 0 \). As a consequence, the estimation error \(\tilde{x} := \hat x - x\) between the observer's state and the actual state decays exponentially to zero with the rate of convergence tuned by \(M\) and \(Q\) \cite{Hamel2017PositionMeasurements}.

To analyze uniform observability (UO), we evaluate the observability Gramian in a convenient block form. In particular, we block-partition the system matrix in~\eqref{eq:state_command_matrix} as
\begin{equation}
A(t) =\begin{bmatrix}A_{11}(t) & A_{12}(t) \\
\mathbf{0}_{3 \times 2} & -\boldsymbol{\omega}^\times(t)
\end{bmatrix}, \label{eq:state_matrix}
\end{equation}
with \[ A_{11}(t) = \begin{bmatrix} 0 & 1 \\ 0 & 0 \end{bmatrix}, \quad
A_{12}(t) = \begin{bmatrix} \mathbf{0}_{3 \times 1} & \mathbf{a}(t) \end{bmatrix}^{\top},\]

and the output matrix of ~\eqref{eq:output_matrix}
\[
C = \begin{bmatrix} C_1 & \mathbf{0}_{1 \times 3} \end{bmatrix},
\quad C_1 = [\,1\ 0\,].
\]
Due to the structure of \(A(t)\) in~\eqref{eq:state_matrix}, the state transition matrix has the block form
\begin{equation}
\Phi(t,\tau) =
\begin{bmatrix} 
\phi_{11}(t,\tau) & \phi_{12}(t,\tau) \\
\mathbf{0}_{3 \times 2} & \phi_{22}(t,\tau)
\end{bmatrix}, \label{eq:state_trans_matrix}
\end{equation}
with $\phi_{11}\in\mathbb{R}^{2\times2}, \phi_{12}\in\mathbb{R}^{2\times3},
\phi_{22}\in\mathbb{R}^{3\times3}.$
From~\eqref{eq:transition_matrix} and~\eqref{eq:state_trans_matrix}, the derivative of the transition matrix is given by
\begin{equation}
\begin{aligned}
\frac{d}{dt} \Phi(t, \tau) &= 
&=\begin{bmatrix}
A_{11}\phi_{11} & A_{11}\phi_{12} + A_{12}\phi_{22} \\
0 & -\boldsymbol{\omega}^\times\phi_{22}
\end{bmatrix}, \label{eq:deriv_trans_matrix}
\end{aligned}
\end{equation}
with \(\phi_{11}(\tau,\tau) = I_2\), \(\phi_{22}(\tau,\tau) = I_3\), and \(\phi_{12}(\tau,\tau) = \mathbf{0}_{2\times3}\).
Since \(A_{11}\) is a constant matrix, we have from~\eqref{eq:deriv_trans_matrix} that 
\begin{equation}
\phi_{11}(t,\tau) = \exp\left(A_{11}(t-\tau)\right) = \begin{bmatrix} 1 & (t-\tau) \\ 0 & 1 \end{bmatrix}\label{eq:ph_11}
\end{equation}
In view of \eqref{eq:state_trans_matrix} and \eqref{eq:output_matrix} and \eqref{eq:observability_gramian}, the Gramian of the LTV system ~\eqref{eq:ltv} is given by  
\begin{equation}
\begin{aligned}
&W(t,t+\delta) =\\
&\frac{1}{\delta}\int_{t}^{t + \delta} \begin{bmatrix} \phi_{11}(s, t)^\top \\
\phi_{12}(s, t)^\top\end{bmatrix} C_1^\top C_1 \begin{bmatrix} \phi_{11}(s,t)&\phi_{12}(s,t)
\end{bmatrix}ds.
\end{aligned}
\label{eq:gramian}
\end{equation}

\begin{lemma} \label{Lemma1}
\textit{Assume that the body-frame linear acceleration \(\mathbf{a}(t)\) and the angular velocity \(\boldsymbol{\omega}(t)\)} are continuous and uniformly bounded. Moreover, assume that the inertial-frame linear acceleration \(a_{\mathcal{I}}(t) := R(t)\mathbf{a}(t)\) is \emph{persistently exciting} (PE) in the sense that there exist  \( \bar{\delta}, \bar{\mu} > 0 \) such that \( \forall t \ge 0 \),
\begin{equation}
\frac{1}{\bar{\delta}}\int_t^{t+\bar{\delta}}a_{\mathcal{I}}(s)a_{\mathcal{I}}(s)^\top\,ds \ge \bar{\mu}. %I_3.
\label{eq:PE_Condition}
\end{equation}
Then the pair $(A(t), C)$ is \emph{uniformly observable}, and the equilibrium \(\tilde x = 0_{5\times1}\) is globally uniformly exponentially stable (GES).
\end{lemma}

\begin{proof}
To show that the system ~\eqref{eq:ltv} is uniformly observable, it suffices to show that there exist \( \bar{\delta}, \bar{\mu} > 0 \) such that ~\eqref{eq:observability_gramian} holds, i.e. \(W(t, t + \bar \delta) > \bar \mu I_n,~ \forall t \geq 0,\) with \(W(t, t + \bar \delta)\) given by~\eqref{eq:gramian}.
Assume, for contradiction, that the proposed system is \emph{not} uniformly observable.  
Then, for every \(\bar \mu > 0\) and  \(\bar \delta > 0\), there exists \(t \ge 0\) such that \(W(t,t+\bar\delta) < \bar\mu I_n \), 
Let \(\{\mu_p\}_{p\in\mathbb{N}}\) be a sequence decreasing to zero with \(\mu_p > 0\), and let 
\(\bar{\delta} > 0\) satisfy the PE condition~\eqref{eq:PE_Condition}.  
Then, there exist sequences \(\{t_p\}\) and \(\{\hat d_p\}\) such that:
\(\hat d_p \in \mathcal{D} := \{\, d \in \mathbb{R}^5 : \|d\| = 1 \,\}\),  
      the \emph{unit sphere in \(\mathbb{R}^5\)},  and 
\(
\hat d_p^\top W(t_p,t_p+\bar{\delta}) \hat d_p < \mu_p,
\quad \forall p \in \mathbb{N}.
\)
Since \(\mathcal{D}\) is compact, there exists a subsequence of \(\{\hat d_p\}\) that converges to some \(d \in \mathcal{D}\), \(\|d\| = 1\).
Letting \(p \to \infty\) and using \(\mu_p \to 0\) yields the convergence relation from  \eqref{eq:gramian}
\begin{equation}
\lim_{p \to \infty} 
\int_0^{\bar{\delta}} 
\big\| C\, \Phi(s,t_p)\, \hat d_p \big\|^2 ds = 0 \label{eq:conv_relation}
\end{equation}
or, by a change of variables with a scalar output function,
\begin{equation}
\lim_{p \to \infty} 
\int_0^{\bar{\delta}} |f_p(s)|^2 ds = 0 \label{eq:conv_relation_f}
\end{equation}
where we define
\[
\begin{aligned}
f_p(t) &:= C \Phi(t+t_p,t_p)\hat d_p \\
&= C_1 \phi_{11}(t+t_p,t_p) d_1 + C_1 \phi_{12}(t+t_p,t_p) d_2,
\end{aligned}
\]
where \(d = [\,d_1^\top, d_2^\top\,]^\top\) with 
\(d_1 \in \mathbb{R}^2\), \(d_2 \in \mathbb{R}^3\).
From~\eqref{eq:transition_matrix},~\eqref{eq:state_matrix}-\eqref{eq:state_trans_matrix},~\eqref{eq:conv_relation}-\eqref{eq:conv_relation_f}, and knowing that \(A_{11}^2 = 0_{2\times2}\), we obtain the successive time-derivatives of \(f_p(t)\), as follow :
\[
\begin{aligned}
%f_p(s) &= C_1 \phi_{11} d_1 + C_1 \phi_{12} d_2, \\[3pt]
f_p^{(1)}(t) &= C_1 A_{11} \phi_{11} d_1 
             + C_1 \left( A_{11}\phi_{12} + A_{12}\phi_{22} \right) d_2, \\
f_p^{(2)}(t) &= C_1 \left( A_{11} A_{12} - A_{12}\boldsymbol{\omega}^\times \right) \phi_{22} d_2
\end{aligned}
\]
Now, using the results of Lemma~A.1 of \cite{Pascal2017},  we deduce:
\[
\lim_{p \to \infty} 
\int_0^{\bar{\delta}} |f_p^{(k)}(s)|^2 ds = 0, 
\quad k=0,1,\dots,2.
\]
The highest derivative yields 
\[f_p^{(2)}(t_p) \to \mathbf{a}^\top(t_p)\big(-\boldsymbol{\omega}^\times(t_p)\big) d_2 \to 0 ~\text{as } p \to \infty.\]
By the PE assumption, this implies \(d_2 = 0\).  
Substituting \(d_2 = 0\) into \(f_p(s)\) and \(f_p^{(1)}(s)\) gives 
\begin{equation} 
C_1\phi_{11}d_1 = 0,  \quad C_1 A_{11}\phi_{11}d_1 = 0, \label{eq:fp_fp_1}
\end{equation}
with  \(d_1 = [\,d_1^{(1)} ~~ d_1^{(2)}\,]^\top\).
From~\eqref{eq:ph_11} we compute \(C_1 \phi_{11}(t,\tau) =  \begin{bmatrix}1 & (t-\tau)\end{bmatrix}\) and \(C_1 A_{11}\phi_{11}(t,\tau) =  \begin{bmatrix}0 & 1\end{bmatrix}\).
By substituting these into~\eqref{eq:fp_fp_1}, we get \(
d_1^{(1)}+d_1^{(2)}(t-\tau) = d_1^{(2)}  = 0\). This implies \(d_1 = 0\).
Therefore \(d = 0\), contradicting \(\|d\| = 1\).  
Hence, the pair \((A(t),C)\) is uniformly observable. This in turn guarantees the global exponential stability of the equilibrium \(\tilde{x} = \boldsymbol{0}_{5\times 1}\)(See~\cite{Hamel2017PositionMeasurements}).  
\end{proof}
Lemma~\ref{Lemma1} imposes a persistent excitation condition on the inertial acceleration, requiring sustained variation in angular velocity and specific acceleration.
\subsection{Attitude Estimation on \( \mathrm{SO}(3) \)}
Once the tilt vector \( \hat{z} \) is estimated from the Riccati observer, full attitude estimation requires an additional known direction to resolve the heading ambiguity. To this end, we incorporate magnetometer measurements \(\mathbf{m}_{\mathcal{B}} \in \mathbb{S}^2 \) defined in \eqref{eq:mag_measurement}. Let \( \hat{R} \in \mathrm{SO}(3) \) denote the estimate of the attitude \( R \). The attitude observer is given by
\begin{equation}
\dot{\hat{R}} = \hat{R} \boldsymbol{\omega}^\times - \sigma_R^\times \hat{R}, \quad \hat{R}(0) \in \mathrm{SO}(3),
\label{eq:attitude_observer}
\end{equation}
where the correction term \( \sigma_R \in \mathbb{R}^3 \) is defined as
\begin{equation}
\sigma_R = k_z (\mathbf{e}_3 \times \hat{R} \hat{z}) + k_m (\bar{\mathbf{m}}_{\mathcal{I}} \times \hat{R} \bar{\mathbf{m}}_{\mathcal{B}}),
\label{eq:attitude_correction}
\end{equation}
with \( \bar{\mathbf{m}}_{\mathcal{I}} = \bar\pi_{\mathbf{e}_3} \mathbf{m}_{\mathcal{I}} \), \( \bar{\mathbf{m}}_{\mathcal{B}} = \bar{\pi}_{\hat z} \mathbf{m}_{\mathcal{B}} \), \( k_z > 0 \), and \( k_m \geq 0 \); where $\bar{\pi}_u := |u|^2 I_3 - uu^\top$ is a regularized projection operator, well defined even when $u=0$.
The use of $\bar{\pi}_{\hat z}$ prevents singularities if $\hat z$ vanishes, and the projected magnetometer vectors ensure that yaw estimation is decoupled from roll and pitch; see \cite{hua2013implementation}. The overall architecture is illustrated in figure~\ref{fig:observer_block_diagram}.

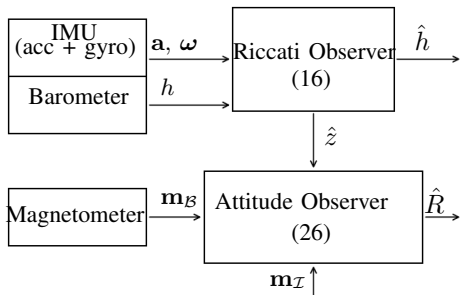
\begin{figure}[ht]
  \centering
  \begin{overpic}[width=0.7\linewidth]{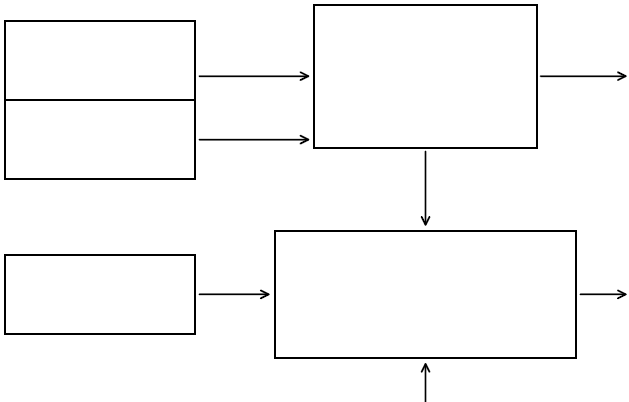}
   %\begin{overpic}[width=0.7\linewidth,grid, tics=10]{Observer_Diagram.png}
    \put(10,56){\small IMU}
    \put(2,53){\small (acc + gyro)}
    \put(5,42){\small Barometer}
    \put(0,16){{\small Magnetometer}}
    \put(32,54){\small \(\mathbf{a},\,\boldsymbol{\omega}\)}
    \put(34,44){\small \(h\)}
    \put(34,20){\small \(\mathbf{m}_{\mathcal{B}}\)}
    \put(50,52){\small Riccati Observer}
    \put(61,46){\small~\eqref{eq:riccati_observer}} 
    \put(90,54){\normalsize \(\hat{h}\)}
    \put(70,33){\normalsize \(\hat{z}\)}
    \put(46,19){\small Attitude Observer}
    \put(61,12){\small~\eqref{eq:attitude_observer}}
    \put(92,18){\normalsize \(\hat{R}\)}
    \put(58,2){\small \(\mathbf{m}_{\mathcal{I}}\)}
  \end{overpic}
  \caption{\small \textit{Two-stage observer architecture: a Riccati observer fuses IMU and barometer data to estimate tilt and height, followed by an \( \mathrm{SO}(3) \) observer combining tilt and magnetometer to estimate \( \hat{R} \).}}
  \label{fig:observer_block_diagram}
\end{figure}

\begin{theorem}\label{Theorem1}
Consider the Riccati observer~\eqref{eq:riccati_observer} and the attitude observer~\eqref{eq:attitude_observer} with correction~\eqref{eq:attitude_correction}. Assume the pair \((A(t), C)\) is uniformly observable according to Lemma~\ref{Lemma1}, then the estimation error \(\tilde{R}=\hat RR^\top, \tilde{x}=x-\hat x\) converge to the set of equilibria \(\mathcal{E} = \mathcal{E}_s \cup \mathcal{E}_u\), where \(\mathcal{E}_s = \{(I_3,0_{5\times1})\}\) and \(\mathcal{E}_u = \left\{(U\Lambda U^\top, 0)\,\middle|\, \Lambda = \mathrm{diag}(1, -1, -1),\, U \in \mathrm{SO}(3) \right\}.\) Moreover, the set \(\mathcal{E}_u\) is unstable, and the singleton \(\mathcal{E}_s\) is almost-globally asymptotically stable (AGAS).
\end{theorem}

\begin{proof}Define the attitude error \( \tilde{R} = \hat{R} R^\top \in \mathrm{SO}(3) \) and tilt error \( \tilde{z} := z -\hat z\). Differentiating \( \tilde{R} \), we obtain
\begin{align}
\dot{\tilde{R}} &= (\hat{R} \boldsymbol{\omega}^\times - \sigma_R^\times \hat{R}) R^\top - \hat{R} \boldsymbol{\omega}^\times R^\top \notag \\
&= -\left( k_z \mathbf{e}_3 \times \hat{R} \hat{z} + k_m \bar{\mathbf{m}}_{\mathcal{I}} \times \hat{R} \bar{\mathbf{m}}_{\mathcal{B}} \right)^\times \tilde{R}.
\notag 
\\&= -\left(k_z\mathbf{e_3} \times \hat{R}z- k_z\mathbf{e_3} \times \hat{R} \tilde{z}+ k_m \bar{\mathbf m}_{\mathcal{I}} \times \hat{R} \bar{\mathbf m}_{\mathcal{B}} \right)^\times \tilde{R}\notag\end{align}

From Lemma~\ref{Lemma1}, it follows that \( \hat z \to z\) exponentially, which implies that \( \bar{\mathbf m}_{\mathcal{B}} \to \bar{\bar{\mathbf m}}_{\mathcal{B}} \), with \( \bar{\bar{\mathbf m}}_{\mathcal{B}} = \bar{\pi}_{z} \mathbf{m}_{\mathcal{B}} \). Moreover, one can show that \(\bar{\mathbf m}_{\mathcal{B}} = \bar{\bar{\mathbf m}}_{\mathcal{B}} + \mathcal{O}(\tilde{z}).\)
Expressing this in terms of \(\tilde x \), and in view of ~\eqref{eq:riccati_observer}, one obtains the closed-loop system:
\begin{subequations} \label{eq:closed_loop}
\begin{align}
\dot{\tilde{R}} &= -\left(
k_z \mathbf{e}_3 \times \hat{R}z 
+ k_m \bar{\mathbf m}_{\mathcal{I}} \times \hat{R} \bar{\bar{\mathbf m}}_{\mathcal{B}} 
+ \mathcal{O}(\tilde x)
\right)^\times \tilde{R}, \label{eq:closed_loop_att} \\
\dot{\tilde{x}} &= (A - KC)\tilde{x}. \label{eq:closed_loop_trans}
\end{align}
\end{subequations}
The above system can be seen as a cascade interconnection of a non-linear system on \(\mathrm{SO}(3)\)~\eqref{eq:closed_loop_att} and the LTV system on \(\mathbb{R}^5\)~\eqref{eq:closed_loop_trans}. To prove the AGAS of the interconnection system, we begin by proving that subsystem~\eqref{eq:closed_loop_att} is AGAS for \(\tilde{x} = 0_{5\times1}\).
From~\eqref{eq:closed_loop_att}, it follows, as shown in~\cite{mahony2008nonlinear}, that the equilibrium sets are 
\(\mathcal{E}_s = \{I_3\}\) and \(\mathcal{E}_u = \left\{(U\Lambda U^\top, 0)\,\middle|\, \Lambda = \mathrm{diag}(1, -1, -1),\, U \in \mathrm{SO}(3) \right\}.\)
The singleton set \(\mathcal{E}_s\) is the stable equilibrium, and the set \(\mathcal{E}_u\) is the set of unstable equilibria (see~\cite[Th.~6.1]{VanGoor2025}). It consists of all 180-degree rotations, each defined by an axis on \(S^2\), which corresponds to a 2D space embedded in the 3D manifold \(\mathrm{SO}(3) \), and thus has measure zero in \(\mathrm{SO}(3) \). It follows that the stable equilibrium \(\tilde{R} = I_3\) is almost globally asymptotically stable for subsystem~\eqref{eq:closed_loop_att}. To complete the proof, we now examine the full interconnection
system. Since the estimation error \(\tilde x\) in~\eqref{eq:closed_loop_trans} evolves independently of \(\tilde{R}\) and is GES from Lemma~\ref{Lemma1}, there exist constants \(\delta, \alpha > 0\) such that \(\tilde x\) satisfies \(
\|\tilde x(t)\| \leq \delta \exp(-\alpha t)\, \|\tilde x(0)\|, \forall t \geq 0.
\)
Thus, \(\tilde x\) remains uniformly bounded, meaning there exists a compact set \(S \subset \mathbb{R}^5\) such that \(\tilde x(t) \in S\) for all \(t \geq 0\). Therefore, according to~\cite[Proposition~2]{Angeli2010}, one can conclude that subsystem~\eqref{eq:closed_loop_att} is almost globally Input-to-State Stable (ISS) with respect to \(\tilde{R} = I_3\) and input \(\tilde x\). Hence, given that \(\tilde x = 0_{5\times 1}\) for system~\eqref{eq:closed_loop_trans} is GES and that subsystem~\eqref{eq:closed_loop_att} with \(\tilde x = 0_{5\times 1}\) is AGAS at \(\tilde{R} = I_3\) and almost globally ISS with respect to \(\tilde x\), it follows from~\cite[Th.~2]{Angeli2004} that the cascaded interconnection system~\eqref{eq:closed_loop} is AGAS at \((\tilde{R}, \tilde x) = (I_3,0_{5\times 1})\).
\end{proof}

\subsection{Discrete-Time Implementation}
The proposed observer is implemented at the IMU sampling period~$T$. 
Over each interval $[t_k,t_{k+1})$, we assume the measured acceleration 
$\mathbf a_k$ and angular velocity $\boldsymbol\omega_k$ are constant 
(see~\cite{Bryne2017}). Let $\Omega_k \doteq \boldsymbol\omega_k^\times$ 
and $\theta_k \doteq \|\boldsymbol\omega_k\|T$. Approximating the discrete process noise by $Q_{d,k}\approx Q_k T$, the transition block associated  with the tilt dynamics satisfies 
$\dot \phi_{22}(t)=-\Omega_k \phi_{22}(t)$ with $\phi_{22}(0)=I_3$, 
which integrates to the incremental rotation
\begin{equation}
\phi_{22,k} 
= I_3 - \frac{\sin\theta_k}{\|\boldsymbol\omega_k\|}\,\Omega_k
+ \frac{1-\cos\theta_k}{\|\boldsymbol\omega_k\|^2}\,\Omega_k^2 .
\label{eq:Phi22}
\end{equation}
Using this result, a first-order discretization of the continuous-time 
state and input matrices in~\eqref{eq:state_command_matrix} yields
\begin{equation}
A_{d,k} \approx
\begin{bmatrix}
1 & T & \tfrac{T^{2}}{2}\,\mathbf a_k^\top \\
0 & 1 & T\,\mathbf a_k^\top \\
0 & 0 & \phi_{22,k}
\end{bmatrix},\qquad
B_{d,k} =
\begin{bmatrix}
\tfrac{T^{2}}{2} \\
T \\
\mathbf 0_{3\times 1}
\end{bmatrix}.
\label{eq:AdBd}
\end{equation}
The resulting discrete-time observer, summarized in 
Algorithm~\ref{algo:Discrete_proposed_obs}, follows the standard 
correction--prediction structure with the above state and input matrices. The attitude observer~\eqref{eq:attitude_observer} is discretized at the 
IMU frequency using exponential-Euler integration on $\mathrm{SO}(3)$~\cite{mahony2008nonlinear} (see line~20 of Algorithm~\ref{algo:Discrete_proposed_obs}). For simplicity, in the present formulation, the magnetometer is assumed to provide 
measurements at the same frequency as the IMU. 
Nevertheless, the framework readily accommodates asynchronous updates, and one may consider the case where the magnetometer operates at a different (typically lower) sampling rate; see \cite{Bryne2017}.

\begin{algorithm}[!t]
\caption{Discrete-Time Implementation of the proposed observer}
\label{algo:Discrete_proposed_obs}
\begin{algorithmic}[1]
\Statex \textbf{Input:} $\hat x_{0|0},\,P_{0|0},\hat R_0$; $g_k,\mathbf a_k,\boldsymbol{\omega}_k$ ; $\mathbf m_{\mathcal B,k}$; $y_{p,k}$; $T$; $Q_k$; $M_k$; gains $k_z,k_m$
\Statex \textbf{Output:} $\hat x_k,\,\hat R_{k},\,$ \(\forall k \in \mathbb{N}_{\geq 1}\)
%\Statex \textbf{\textcolor{blue}{/* Prediction Step: */}}
\For{each time \(k \geq 1\)}
\If{IMU data \(\mathbf a_k\), \(\omega_k\) is available}
    \State $\Omega_k \gets \boldsymbol{\omega}_k^\times$,\quad $\theta_k \gets \|\omega_k\|\,T$, \quad $Q_{d,k}\gets Q_k\,T$;
    \State {$\phi_{22,k} \gets $~\eqref{eq:Phi22};}
    \State $A_{d,k}, B_{d,k} \gets$~\eqref{eq:AdBd};
    \State $\hat x_{k+1|k} \gets A_{d,k}\hat x_{k|k} + B_{d,k} g_k$;
    \State $P_{k+1|k} \gets A_{d,k}P_{k|k}A_{d,k}^\top + Q_{d,k}$
\EndIf
%\Statex \textbf{\textcolor{blue}{/*Update Step: */}}
\If {barometer data is available}
    \State $C_k \gets$\eqref{eq:output_matrix}; $M_{d,k}\gets M_k$
    %\State $S_{k+1} \gets C_kP_{k+1|k} C^\top + M_{d,k}$;
    \State $K_{k+1} \gets P_{k+1|k} C_k^\top (C_kP_{k+1|k} C_k^\top + M_{d,k})^{-1}$
    \State \(\hat x_{k+1|k+1} \gets \hat x_{k+1|k}+ K_{k+1}\left(y_{p,k+1}-C_k \hat x_{k+1|k}\right)\)
    \State \(P_{k+1|k+1} \gets (I_5-K_{k+1}C_k)P_{k+1|k}\)%(I_5-K_{k+1}C_k)^\top + K_{k+1}M_{d,k}K_{k+1}^\top$
\Else
    \State $\hat x_{k+1|k+1}\gets \hat x_{k+1|k}$; $P_{k+1|k+1}\gets P_{k+1|k}$.
\EndIf
\State $P_{k+1|k+1}\gets \frac{1}{2}(P_{k+1|k+1} + P_{k+1|k+1}^{\top})$
%\Statex \textbf{\textcolor{blue}{/*Attitude estimation: */}}
\State $\bar{\mathbf m}_{\mathcal I} \gets \bar\pi_{\mathbf e_3}\mathbf m_{\mathcal I}$;\quad $\bar{\mathbf m}_{\mathcal B,k} \gets \bar\pi_{\hat z_k}\mathbf m_{\mathcal B,k}$
\State $\sigma_{R,k} \gets k_z(\mathbf e_3 \times \hat R_k \hat z_k) + k_m(\bar{\mathbf m}_{\mathcal I} \times \hat R_k \bar{\mathbf m}_{\mathcal B,k})$
\State $\hat R_{k+1} \gets \hat R_k \exp\!\big((\boldsymbol{\omega}_k - \hat R_k^\top \sigma_{R,k})^\times T\big)$
\EndFor
\end{algorithmic}
\end{algorithm}
\section{SIMULATION RESULTS}
To evaluate the performance of the proposed Barometer-Aided cascade observer, we conduct a simulation of a vehicle illustrated in figure~\ref{fig:veh_baro_imu} equipped with a barometer-IMU system moving in 3D space. The ground truth angular velocity and the inertial altitude are generated as follows:
\[
\boldsymbol{\omega}(t) =\begin{bmatrix} 0.4\sin(0.5t)\\ 0.5\sin(0.3t + \pi/4)\\0.3 \sin(0.7 t + \pi/3)\end{bmatrix}, \quad h(t) = -5\sqrt{3}\sin(2t)/4,
\]
and the body-frame linear acceleration is generated as follows:
\[
\mathbf{a}(t) =R^\top\big(
\begin{bmatrix}
%- 0.04\cos(0.2t)\\-(0.04) \sin(0.4t)\\0.04\sqrt{3} \sin(0.4t)
- \cos(t)&- \sin(2t)&5\sqrt{3}\sin(2t)
\end{bmatrix}^{\top}-\mathbf{g}\big).
\]

A Monte Carlo simulation with 50 runs is performed, where the initial estimates are randomly sampled from Gaussian distributions. The initial state estimates are normally distributed around $\hat x(0) = \begin{bmatrix} 5 & 5 & \hat R(0)^\top \mathbf{e}_3\end{bmatrix}^{\top}$ with standard deviation $\boldsymbol{\sigma}_{x} = \begin{bmatrix}8&8&0.5&0.5&0.5\end{bmatrix}^\top$ . The initial orientation estimates are normally distributed around $\hat R(0)$ which corresponds to the initial angles \(\begin{bmatrix}60^\circ; -30^\circ; 45^\circ\end{bmatrix};\)  with a standard deviation of $104^\circ$ per axis. The parameters are set as follows: \(Q = 10I_5 ,~k_z=80,~k_m=25\).
The constant vector $\mathbf{m}_\mathcal{I}$ is set to $\mathbf{m}_\mathcal{I} =\begin{bmatrix} 1/\sqrt{2} & 0& 1/\sqrt{2} \end{bmatrix}^{\top}$. 
The IMU measurements $\mathbf{a}(t)$, $\boldsymbol{\omega}(t)$, and $\mathbf{m_B}(t)$ in body frame, sampled at $200~(Hz)$, are corrupted with zero-mean Gaussian noise with standard deviations $0.05$, $0.05$, and $0.02$, respectively. The barometer sampled at $5~(Hz)$ is corrupted with zero-mean Gaussian noise with standard deviation $\sigma_b^2 = 0.001$. Figures~\ref{fig:Tilt_Component_Error}, and~\ref{fig:Euler_Angles_Att_Error} illustrate the reduced attitude components and error (\(\left\| \hat z - z \right\|\)), and the Euler angles and attitude estimation error (\(\operatorname{trace}\!\left(I_{3} - R\hat{R}^\top\right)\)).
The simulation results demonstrate that all estimation errors decrease and converge to zero, and that both the tilt and altitude estimates remain bounded with fast convergence for all initial conditions. These results confirm the expected asymptotic convergence of the proposed observer.

\begin{figure}[!t]
    \centering
    \includegraphics[width=0.75\linewidth]{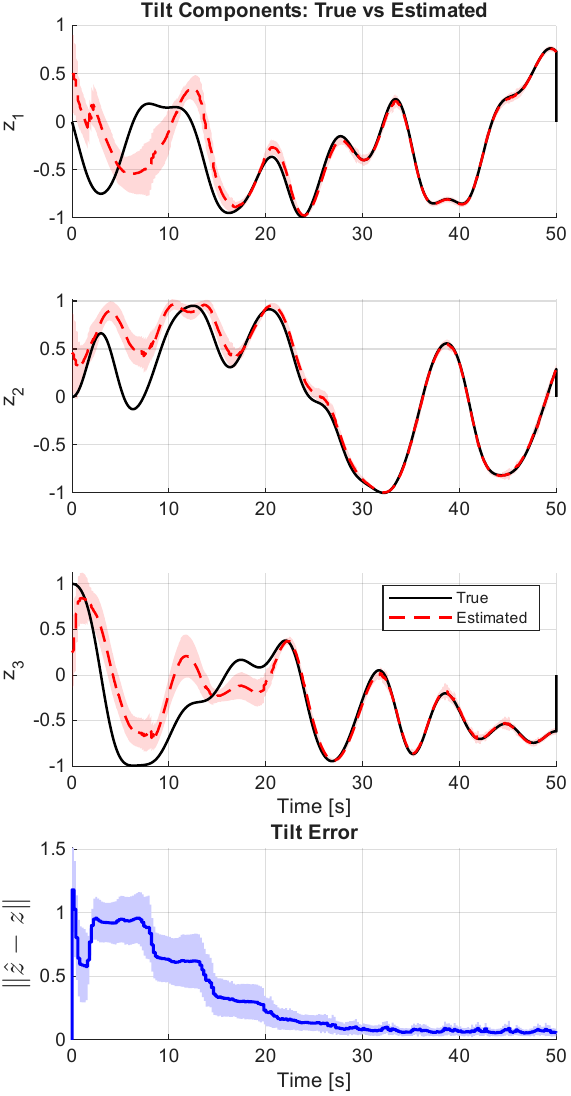}
    \caption{Tilt components and tilt error.}
    \label{fig:Tilt_Component_Error}
\end{figure}

\begin{figure}[!t]
    \centering
    \includegraphics[width=.75\linewidth]{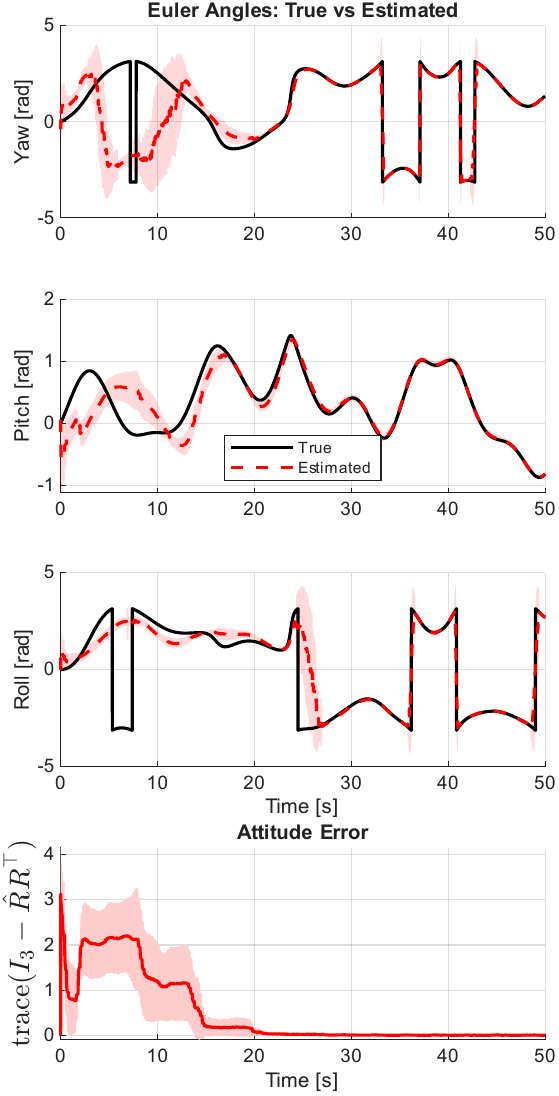}
    \caption{Euler angles and attitude error.}
    \label{fig:Euler_Angles_Att_Error}
\end{figure}

\section{CONCLUSION}
In this work, we proposed a cascade observer for barometer-aided attitude estimation that combines a continuous-discrete Riccati observer for altitude and vertical dynamics with a nonlinear observer on $\mathrm{SO}(3) $ for full attitude estimation. Theoretical analysis (Lemma~\ref{Lemma1} and Theorem~\ref{Theorem1}) established almost-global convergence of the interconnection error dynamics under persistently exciting motion.
Monte Carlo simulation results confirm the theoretical results, showing that the proposed observer achieves consistent convergence of both reduced attitude and full attitude errors to zero, even under initial state and orientation uncertainties.
Future work will focus on extending the approach to account for biased IMU measurements.
\bibliographystyle{IEEEtran}
\bibliography{root}

\end{document}